\documentclass{article}
 
\usepackage{iclr2027_conference,times}

\usepackage{hyperref}
\usepackage{url}

\title{Overcoming Scaling Limits in On-Policy Self-Distillation for LLM Reasoning}
\author{Md Ismail Hossain \\
North South University \\
\And  
Humaira Kousar \\
KAIST \\
\And
Isidora Chara Tourni \\
Andria Labs \\
\thanks{Correspondence: \texttt{ismail.hossain2018@northsouth.edu}, 
\texttt{humairakousar32@kaist.ac.kr}, \texttt{isidora@andrialabs.ai}.}}

\usepackage{multirow}
\usepackage{booktabs}
\usepackage{multirow}
\usepackage{graphicx}
\usepackage{wrapfig}
\usepackage[table]{xcolor}
\usepackage{amsmath}
\usepackage{amsthm}
\usepackage{amssymb}
\usepackage{algorithm}
\usepackage{algpseudocode}
\usepackage{enumitem} 
\usepackage{tcolorbox}
\usepackage{xcolor}

\newtheorem{proposition}{Proposition}
 
\newtheorem{proposition*}{Proposition}
 
\usepackage{tcolorbox}

\tcbset{
  abstractbox/.style={
    colback=gray!5,       % Very light gray background
    colframe=gray!50,     % Subtle border color
    arc=2mm,              % Slightly rounded corners
    boxrule=0.5pt,        % Border thickness
    left=12pt, right=12pt, top=10pt, bottom=10pt % Inner padding
  }
}

\newtcolorbox{takeawaybox}[1][]{
  colback=gray!6!white,
  colframe=gray!60!black,
  fonttitle=\bfseries,
  title=Section Takeaways,
  arc=2mm,
  boxrule=0.8pt,
  left=3mm, right=3mm, top=2mm, bottom=2mm,
  #1
}
 
\definecolor{scaffoldblue}{HTML}{0B4F9C}
\definecolor{contextorange}{HTML}{E4572E}
\definecolor{headergray}{gray}{0.90}
\definecolor{headergray}{HTML}{F2F4F7}
\definecolor{ourshighlight}{HTML}{E8F1F5}
\definecolor{diffhighlight}{HTML}{FDF2E9}
\definecolor{bestcolor}{HTML}{0D47A1}
\definecolor{diffcolor}{HTML}{C0392B}

\iclrfinalcopy 
\begin{document}

\maketitle

\begin{abstract}
 
\begin{tcolorbox}[]
\small
\vspace{1mm}

On-policy self-distillation (OPSD) trains a student to match a privileged teacher distribution along its own sampled trajectory. Standard OPSD applies this supervision to unverified student rollouts while conditioning the teacher on privileged context, typically a reference solution. We separate these roles in a factorial analysis and find that scaffold correctness has a stronger effect on downstream accuracy than context correctness. Unverified scaffolds create an imitation gap because the teacher can use information unavailable to the student. This gap shrinks with model scale, yet OPSD continues to supervise mostly unverified trajectories. In contrast, verified scaffolds remain effective even when the teacher is conditioned on the student’s own unsuccessful rollout. Based on this finding, we introduce OASIS, which retains the OPSD objective but supervises mostly verified by label on-policy trajectories and replaces written solutions with unverified model-generated attempts as the teacher context. OASIS therefore requires only final-answer labels. Across Qwen3-1.7B, 4B, and 8B on AIME 2024, AIME 2025, and HMMT 2025, OASIS improves over the base model by 3.2--3.8 points on average, while OPSD's gain falls from 3.05 points at 1.7B to 0.14 at 8B. At 8B, OASIS improves over OPSD by 3.05 points, showing that verified on-policy scaffolds preserve the effectiveness of self-distillation as models scale.

\end{tcolorbox}
 
\end{abstract}

\section{Introduction}
\label{sec:intro}

On-policy self-distillation (OPSD) provides an efficient approach to improving mathematical reasoning in language models \citep{zhao2026selfdistilled}. In OPSD, a single model is used in two conditioning settings. The student receives only the problem and samples a solution, while the teacher is additionally given a reference solution and evaluates the student-generated trajectory at the token level. The student is then trained to match the teacher's token distributions along its own sampled trajectory. Since the supervision is applied to trajectories generated by the student itself, OPSD reduces the distribution mismatch associated with sequence-level distillation \citep{hinton2015distilling, kim2016sequence}, while retaining the dense supervision of token-level distillation. It is also more token-efficient than reinforcement learning for mathematical reasoning \citep{shao2024deepseekmath, zhao2026selfdistilled} and builds on on-policy distillation \citep{ross2011reduction, gu2024minillm}.

Two aspects of OPSD remain less understood. First, it requires a worked solution for each training problem, which limits its applicability to datasets that provide only verifiable answers. Second, it is not yet clear what aspects of the teacher signal are most important for learning. Recent analyses have approached the latter question from the teacher side. \citet{shen2026purified} decompose the teacher signal into a component induced by the reference solution and a component that depends only on the question, while \citet{ichihara2026privileged} show that a reference solution written for a different problem can still provide useful supervision. The specific reference solution may therefore matter less than the standard formulation implies \citep[see also][]{zelikman2022star, wang2022self}. This leaves a complementary question largely unexplored: does it matter \emph{where} along the student's trajectory the supervision is applied?
 
\begin{figure}[t]
    \centering
    \includegraphics[width=\textwidth]{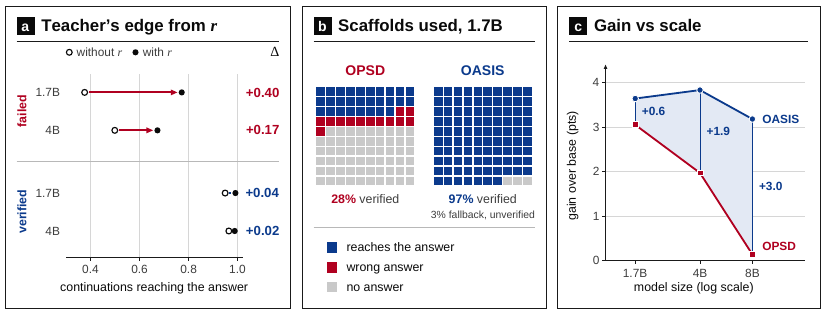}
    \vspace{-0.5cm}
    \caption{
    \textbf{The teacher's advantage is concentrated on failed trajectories and decreases with model scale.}
    \textbf{(a)} Fraction of continuations reaching the correct answer from prefixes of student rollouts, with and without the gold reference $r$. \textbf{(b)} Scaffolds supervised at 1.7B: only 28\% of OPSD's reach a verified answer, and most of the rest are truncated. \textbf{(c)} Average gain over the base model.
    }
    \label{fig:teaser}
\end{figure}

We investigate this question with a recovery experiment: we truncate a student rollout and let the model continue from the same prefix under different contexts. On rollouts that end in a wrong answer, the gold-context teacher recovers the correct answer in 77\% of continuations versus 37\% for the student, while unrelated or empty references give no improvement. On rollouts that already end correctly, the gap is small (99\% versus 95\%). The teacher's advantage is therefore concentrated on failed trajectories and relies on information unavailable to the student at test time (Figure~\ref{fig:teaser}a, Table~\ref{tab:recovery}). This advantage shrinks with scale. From 1.7B to 4B, the student's own recovery on failed prefixes rises from 0.37 to 0.50 and the teacher's advantage falls by more than half, yet about 70\% of OPSD updates still fall on trajectories that never reach a correct answer. In our experiments, OPSD improves over the base model by 3.05, 1.98, and 0.14 points at 1.7B, 4B, and 8B.

These observations motivate OASIS (On-policy Alignment via Scaffold-Isolated Supervision), which keeps the OPSD objective unchanged. OASIS samples $K$ rollouts per problem, verifies their final answers \citep{shao2024deepseekmath}, and applies the OPSD loss along the shortest verified rollout. Problems without a verified rollout contribute no signal. Because the teacher signal on verified trajectories is largely insensitive to the context (Figure~\ref{fig:signal}b), OASIS conditions the teacher on a distinct same-problem rollout, typically one of the student's own unverified attempts, and otherwise another verified rollout, requiring only final-answer verification rather than written solutions. Across Qwen3-1.7B, 4B, and 8B, OASIS improves over OPSD by 0.59, 1.86, and 3.05 points on average across AIME 2024, AIME 2025, and HMMT 2025, and its gain over the base model stays stable with scale (Table~\ref{tab:main}, Figure~\ref{fig:teaser}c).

Our main contributions are the following.

\begin{itemize}
\item We characterize the irreducible imitation gap in OPSD when the reference is unavailable to the student (Section~\ref{sec:gap}).

\item We separate two roles that standard OPSD couples: the trajectory scaffold determines the prefixes receiving supervision, while teacher context determines the target distribution at those prefixes. We introduce signal-replay and recovery probes to analyze these roles separately.
 
\item Across Qwen3-1.7B and Qwen3-4B, the probes show that the gold-context recovery advantage is substantially larger on prefixes from non-verified trajectories, whereas teacher targets on outcome-verified trajectories are relatively insensitive to several alternative contexts. 

\item We introduce OASIS, which selects a short outcome-verified scaffold from on-policy-generated candidates and pairs it with a distinct self-generated context. Across three Qwen3 sizes and three competition-mathematics benchmarks, OASIS improves mean Avg@12 over OPSD while replacing written reasoning traces with final-answer labels, at increased rollout-generation cost.

\end{itemize}

\section{Related Work}
\label{sec:related}

\paragraph{On-Policy Self-Distillation and Privileged Context.}
Classical knowledge distillation trains a student against a teacher on a fixed corpus or teacher-generated sequences \citep{hinton2015distilling, kim2016sequence, romero2015fitnets}. On-policy variants instead evaluate the teacher on samples from the student, reducing the distribution mismatch between training and inference \citep{agarwal2024policy, gu2024minillm}. \citet{zhao2026selfdistilled} use a single model as both teacher and student, conditioning only the teacher on a privileged context. \citet{shen2026purified} decompose the teacher signal into reference-induced and question-conditioned components to purify long-chain reasoning supervision, and \citet{ichihara2026privileged} show that a solution to a different problem can still improve the student. Our measurements agree: on verified scaffolds, an empty reference block is about as close to the gold-context teacher as an unrelated solution (Figure~\ref{fig:signal}b). Recent extensions to on-policy self-distillation refine token-level weighting by addressing teacher uncertainty or trajectory dynamics. For instance, EGRSD \citep{ke2026respecting} introduces a teacher-entropy confidence gate to mitigate high-entropy token noise, while DASH \citep{hou2026dash} employs divergence-adaptive supervision horizons to capture path-dependent discrepancy dynamics. These works vary based on what the teacher observes. We instead ask on which states of the student's trajectories its supervision is imitable.

\paragraph{Learning from Privileged Information.}
The setting in which an expert observes information unavailable to the learner has long been studied as learning using privileged information \citep{vapnik2009lupi}. In sequential decision-making, the related imitation gap arises when a policy is trained to imitate an expert with richer observations and is then deployed without that privileged information \citep{weihs2021bridging, swamy2022sequence}. Asymmetric
actor-critic methods similarly allow privileged information during training while restricting the deployed policy to its available observations \citep{pinto2018asymmetric}. OPSD fits this setting: when the teacher's action depends on the unobserved reference, the student is pushed toward a context-averaged target, leaving the irreducible loss of Proposition~\ref{prop:mixture}.

\paragraph{Verification-Based Reasoning Training.}
Outcome verification is widely used to train reasoning models from their own samples. Rejection-sampling fine-tuning retains successful trajectories as training data
\citep{yuan2023scaling}, while STaR and expert iteration iteratively use verified or self-generated solutions to improve the model \citep{zelikman2022star, anthony2017expert}. Process supervision instead evaluates intermediate reasoning steps and can provide finer-grained credit assignment, at the cost of step-level annotations or reward models \citep{uesato2022solving, lightman2024let}.
Reinforcement learning with verifiable rewards uses the same type of outcome signal as a sequence-level reward \citep{schulman2017ppo, shao2024deepseekmath}. In OASIS, the verified trajectory is neither a supervised target nor a scalar reward: verification only decides which on-policy scaffold receives the teacher's token-level distribution. The comparison with GRPO in Section~\ref{sec:main} is complementary, since both use the same verifier.

\section{Diagnosing On-Policy Self-Distillation}
\label{sec:diagnosis}

\subsection{Preliminaries}
\label{sec:setup}

\paragraph{On-Policy Self-Distillation (OPSD).}
Let $x$ denote a problem and $\pi_\theta$ an autoregressive policy. OPSD uses the same model parameters in two conditioning settings. The student receives only the problem and samples a trajectory
\begin{equation}
y = (y_1,\dots,y_L) \sim \pi_\theta(\cdot \mid s(x)),
\end{equation}
where $s(x)$ denotes the student prompt. The teacher receives the same problem together with a privileged context $r$, typically the written gold solution. This context is provided through a teacher prompt $u(x,r)$. At each prefix $y_{<t}$ of the student-generated trajectory, the two branches produce
\begin{equation}
p_t = \pi_\theta^{(T)}\!\left(\cdot \mid s(x), y_{<t}\right),
\qquad
q_t = \pi_{\theta_0}^{(T)}\!\left(\cdot \mid u(x,r), y_{<t}\right),
\label{eq:branches}
\end{equation}
where $\pi^{(T)}$ denotes the next-token distribution at temperature $T$.

The student is trained to match teacher distributions at these prefixes using a forward KL divergence. Following OPSD, each vocabulary contribution is clipped at a constant $\kappa$, and the loss is averaged over supervised tokens in a batch $\mathcal{B}$:
\begin{equation}
\mathcal{L}(\theta) =
\frac{1}{\sum_{(x,y)\in\mathcal{B}} |y|}
\sum_{(x,y)\in\mathcal{B}}
\sum_{t=1}^{|y|}
\sum_{a\in\mathcal{V}}
\min\left(
q_t(a)\log\frac{q_t(a)}{p_t(a)},
\kappa
\right).
\label{eq:opsd}
\end{equation}

We refer to $y$ as the \emph{scaffold} (defining where supervision is applied) and $r$ as the \emph{context} (defining information available to the teacher). Standard OPSD pairs an unverified student scaffold with a gold reference context.
\paragraph{Offline Diagnostic Protocols.}
To analyze OPSD without training overhead, we evaluate two diagnostic setups: (i) a \emph{signal replay test} that replays student and teacher distributions along recorded scaffolds to measure next-token divergence under varied contexts, and (ii) a \emph{truncated continuation benchmark} that cuts student rollouts at fraction $f \in \{0.1, \dots, 0.9\}$ of their length and measures outcome recovery under different teacher contexts (full details in Appendix~\ref{app:probes}).

\subsection{Supervision Allocation and Waste}
\label{sec:where}

Tracking trajectory outcomes across training logs reveals a major efficiency bottleneck. At Qwen3-1.7B and Qwen3-4B, only 28.3\% and 30.6\% of training rollouts, respectively, produce a correct answer, while 62.0\% and 64.4\% reach the token cap without producing an answer (Table~\ref{tab:budget}). Thus, 71.7\% and 69.4\% of rollouts end without a verified answer. Truncated rollouts also contain neither an answer token nor a stop token, so they provide no supervision for completing the solution. Over 100 training steps, rollouts become longer while the fraction that verifies decreases from 0.333 to 0.229 at 1.7B and from 0.394 to 0.218 at 4B (Figure~\ref{fig:dynamics}).

\begin{figure}[t]
\centering
\includegraphics[width=\textwidth]{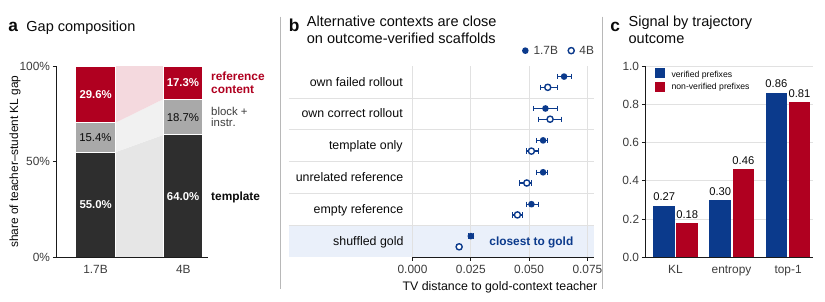}
\vspace{-0.5cm}
\caption{\textbf{Anatomy of the teacher signal.}
\textbf{(a)} Sequential attribution of the teacher--student KL gap.
\textbf{(b)} Total-variation distance from the gold-context teacher for alternative contexts, with 95\% bootstrap intervals.
\textbf{(c)} Teacher statistics on verified and non-verified prefixes.}
\label{fig:signal}
\end{figure}

\subsection{Anatomy of the Teacher Signal}
\label{sec:anatomy}

Much of the teacher--student KL gap comes from formatting rather than reference content. At 1.7B, the thinking-mode template accounts for 55.0\% of the gap and the reference block and its instructions for 15.4\%, leaving 29.6\% for the reference content itself, which falls to 17.3\% at 4B (Figure~\ref{fig:signal}a). On verified scaffolds the teacher is also weakly sensitive to its context: the student's own failed attempt gives a total-variation distance of 0.065 from the gold-context teacher, and unrelated, empty, and template-only contexts give 0.056, 0.051, and 0.056 (Figure~\ref{fig:signal}b). A line-shuffled gold reference is closest (0.025), suggesting the teacher relies on the reference content more than on the order of its reasoning \citep{ichihara2026privileged}.

\begin{table}[t]
\centering
\caption{\textbf{The teacher's advantage is concentrated on failed trajectories.} Fraction of continuations that reach the correct answer after resuming from a prefix of a student rollout, averaged over cut fractions $f\in\{0.1,\ldots,0.9\}$. The student receives no context. $\Delta$ is the difference from the student-alone condition of the same run. $^{*}$Restricted to problems where the answer does not appear elsewhere in the redacted reference (failed: 69 at 1.7B, 72 at 4B; verified: 37 at 1.7B, 52 at 4B).}
\label{tab:recovery}
\small
\resizebox{10cm}{!}{
\begin{tabular}{llcccc}
\toprule
\rowcolor{headergray}
& & \multicolumn{2}{c}{\textbf{Qwen3-1.7B}} & \multicolumn{2}{c}{\textbf{Qwen3-4B}} \\
\cmidrule(lr){3-4}\cmidrule(lr){5-6}
\rowcolor{headergray}
\textbf{Prefix} & \textbf{Context} & \textbf{Recovers} & $\boldsymbol{\Delta}$ & \textbf{Recovers} & $\boldsymbol{\Delta}$ \\
\midrule
\multirow{7}{*}{Failed}
& Gold reference              & 0.77 & \textbf{+0.40} & 0.67 & \textbf{+0.17} \\
& Answer only                 & 0.61 & +0.23          & 0.59 & +0.09 \\
& Gold, answer removed        & 0.59 & +0.22          & 0.57 & +0.08 \\
& Gold, answer removed$^{*}$  & 0.51 & +0.13          & 0.50 & +0.01 \\
& Unrelated reference         & 0.32 & $-0.06$        & 0.37 & $-0.13$ \\
& Empty reference             & 0.37 & $-0.01$        & 0.41 & $-0.09$ \\
& \textit{Student alone}      & \textit{0.37} & ---   & \textit{0.50} & --- \\
\midrule
\multirow{6}{*}{Verified}
& Gold reference              & 0.99 & \textbf{+0.04} & 0.99 & \textbf{+0.02} \\
& Answer only                 & 0.98 & +0.03          & 0.98 & +0.02 \\
& Gold, answer removed        & 0.95 & +0.00          & 0.95 & $-0.01$ \\
& Unrelated reference         & 0.91 & $-0.04$        & 0.93 & $-0.03$ \\
& Empty reference             & 0.94 & $-0.01$        & 0.96 & $-0.01$ \\
& \textit{Student alone}      & \textit{0.95} & ---   & \textit{0.97} & --- \\
\bottomrule
\end{tabular}}
\end{table}

\subsection{Privileged Teacher Advantage}
\label{sec:advantage}

On failed prefixes at 1.7B, the gold-context teacher recovers a correct answer in 77\% of cases versus 37\% for the student ($\Delta = +0.40$ Table~\ref{tab:recovery}), while empty and unrelated references give 37\% and 32\%, so the advantage comes from the privileged content. The final answer alone gives 61\% and the written reasoning without the answer 51\% (excluding cases where the answer reappears), so both contribute, and neither is available to the student at inference time. On verified scaffolds the gap is small ($\Delta = +0.04$ at 1.7B and $+0.02$ at 4B). On failed states, the teacher has higher entropy (0.461 vs.\ 0.297), assigns lower probability to the token sampled by the student (0.775 vs.\ 0.834), and produces a lower KL signal (0.181 vs.\ 0.274) than on verified states (Table~\ref{tab:states}). Thus, OPSD applies most of its supervision on failed states, where the teacher's predictions are both more context-dependent and less concentrated.

\subsection{The Privileged Imitation Gap}
\label{sec:gap}

Because the student cannot observe $r$, it cannot in general match teacher distributions that depend on privileged context. Under forward-KL distillation, the optimal student instead matches the context-averaged teacher distribution.

\begin{proposition}[Mixture Optimum and Irreducible Loss]
\label{prop:mixture}
Let $z=(x,y_{<t})$ denote the student observation and let $R\sim P(r\mid z)$ denote the privileged context. Minimizing
$\mathbb{E}_{R}\left[\mathrm{KL}\!\left(q(\cdot\mid z,R)\parallel p(\cdot\mid z)\right)\right]$
over $p$ yields
$p^\star(\cdot\mid z)=\mathbb{E}_{R\mid z}\left[q(\cdot\mid z,R)\right]$.
The minimum objective value equals the conditional mutual information $I(A;R\mid Z=z)$, and
$H(p^\star(\cdot\mid z))\geq\mathbb{E}_{R\mid z}H(q(\cdot\mid z,R))$.
\end{proposition}

(Proof in Appendix~\ref{app:proofs}.) Proposition~\ref{prop:mixture} formalizes the constraint imposed by privileged information. Whenever teacher actions retain information about $r$ that is unavailable to the student, the corresponding KL loss cannot be eliminated by updating the student alone. This connects OPSD to the broader imitation-learning setting in which the learner lacks information available to the expert \citep{swamy2022sequence}. This predicts that distillation on states with a large privileged advantage can increase the student's uncertainty. Section~\ref{sec:training-behaviour} examines this during training.

\subsection{Scaling Bottlenecks and KL Clipping}
\label{sec:scaling}

The privileged advantage shrinks with scale: reference content accounts for 29.6\% of the KL gap at 1.7B but 17.3\% at 4B (Figure~\ref{fig:signal}a), and $\Delta$ on failed prefixes falls from $+0.40$ to $+0.17$ (Table~\ref{tab:recovery}). OPSD nevertheless keeps about 70\% of its updates on unverified trajectories, and its gain over the base model falls from 3.05 to 1.98 to 0.14 points across 1.7B, 4B, and 8B (Table~\ref{tab:main}). The per-entry KL clip $\kappa$ in Eq.~\ref{eq:opsd} is a second, independent issue. Proposition~\ref{prop:clip} (Appendix~\ref{app:clip}) shows that clipping an entry removes the teacher's upward pull on that token and reverses the sign of its logit gradient. At $\kappa = 0.05$, the clip removes 97\% of positive KL mass on verified scaffolds and 94\% on failed ones, affecting 29--32\% of tokens (Figure~\ref{fig:clipping}).

\begin{takeawaybox}
The diagnosis shows that teacher advantage is concentrated on failed trajectories, while the recovery gap is small on verified trajectories. This advantage shrinks with scale, alongside diminishing OPSD gains from 1.7B to 8B. These findings reveal a scaling bottleneck and motivate distillation on verified trajectories.
\end{takeawaybox}

\section{OASIS: distilling where the teacher is imitable}
\label{sec:method}

Supervision is most useful when the teacher's behaviour can be reproduced from information available to the student. The corresponding quantity, $I(A;R\mid Z)$ in Proposition~\ref{prop:mixture}, cannot be evaluated during training, so OASIS uses outcome verification as an empirical proxy: it restricts supervision to verified on-policy trajectories and builds the teacher context from the same rollouts.

\begin{figure}[t]
\centering
\includegraphics[width=\textwidth]{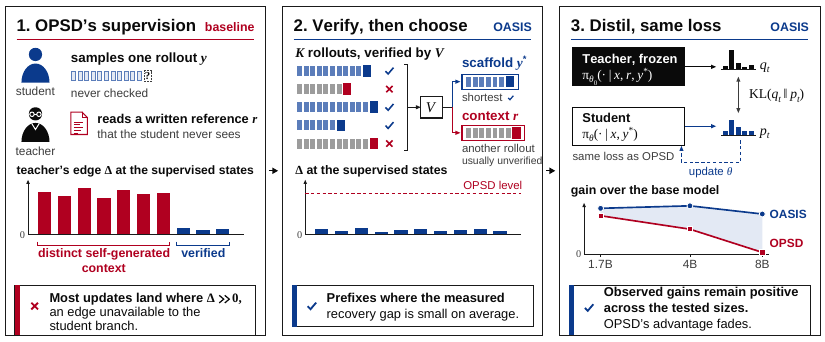}
\vspace{-0.4cm}
\caption{\textbf{OPSD versus OASIS.} \textbf{(1)} OPSD distills along a single unverified rollout $y$, so most updates land where the teacher's edge from the hidden reference $r$ is large. \textbf{(2)} OASIS samples $K$ rollouts, verifies them, and takes the shortest verified one as scaffold $y^*$ and another same-problem rollout as context $r$, typically unverified as context $r$, where the teacher's edge is small. \textbf{(3)} The student is trained against the frozen teacher with the unchanged OPSD loss.}
\label{fig:method}
\end{figure}

\subsection{Verification as a criterion for supervision}
\label{sec:criterion}

Let $V(x,y)\in\{0,1\}$ denote a task verifier that evaluates the final answer of a trajectory. The expected OPSD objective can be decomposed according to the verifier outcome:
\begin{equation}
\mathbb{E}_{y\sim\pi_{\theta}}\big[\mathcal{L}(\theta;r,y)\big] = 
\underbrace{\Pr(V{=}1) \cdot \mathbb{E}\big[\mathcal{L}\mid V{=}1\big]}_{\text{verified trajectories}}
+
\underbrace{\Pr(V{=}0) \cdot \mathbb{E}\big[\mathcal{L}\mid V{=}0\big]}_{\text{unverified trajectories}}.
\label{eq:split}
\end{equation}

The two groups differ sharply in privileged advantage ($\Delta=+0.04$ versus $+0.40$ at 1.7B,Table~\ref{tab:recovery}), yet unverified trajectories receive about 72\% of OPSD updates at 1.7B (Table~\ref{tab:budget}).

OASIS retains supervision from the verified trajectories:
\begin{equation}
\mathcal{L}_{\textsc{oasis}}(\theta) =
\mathbb{E}_{x}\Big[
\mathbb{1}\big[\mathcal{Y}^{+}(x)\neq\emptyset\big] \cdot
\mathcal{L}\big(\theta;r(x),y^\star(x)\big)
\Big],
\qquad
\mathcal{Y}^{+}(x) = 
\{y\in\mathcal{Y}(x):V(x,y)=1\},
\label{eq:oasis}
\end{equation}
where $\mathcal{Y}(x)$ contains $K$ rollouts sampled from the current policy. Except for rare all-masked micro-batches, this indicator restricts supervision to problems with at least one verified rollout. The fallback rule used in that rare case is given in Appendix~\ref{app} (paragraph ``Masking''). The loss, teacher, prompts, clipping constant, temperature, and optimizer are identical to OPSD. Only the distribution of supervised trajectories changes, and the scaffold remains on policy. Unlike rejection-sampling fine-tuning \citep{yuan2023scaling, zelikman2022star}, the verified trajectory is not the training target: the teacher distribution is. %Section~\ref{sec:predictions} compares these factors directly.

\subsection{Scaffold: the shortest verified trajectory}
\label{sec:scaffold}

For each problem, we draw $K$ rollouts using the student prompt and the sampling settings of OPSD, verify their final answers, and select
\begin{equation}
y^\star(x) =
\operatorname*{argmin}_{y\in\mathcal{Y}^{+}(x)} |y|.
\label{eq:shortest}
\end{equation}

Verification decides whether a problem contributes supervision. Length only chooses among verified trajectories. A verified trajectory contains a completed answer and a stopping point that truncated rollouts lack (Section~\ref{sec:where}), and the shortest one limits the weight of unnecessarily long trajectories. We treat this rule as an empirical choice (Appendix~\ref{app:ablations}). %51.7\% of problems are supervised at 1.7B and 50.8\% at 4B (Table~\ref{tab:budget}).

\subsection{Context: a distinct same-problem rollout} 
\label{sec:context}

On verified scaffolds the teacher is relatively insensitive to its context (Section~\ref{sec:anatomy}), consistent with the student already recovering 95\% of verified prefixes without privileged context. We therefore construct the context from another rollout of the same problem:
\begin{equation}
r(x)
\in
\mathcal{Y}(x)\setminus\mathcal{Y}^{+}(x),
\qquad
r(x)
=
\operatorname*{arg\,min}_{y\in
\mathcal{Y}(x)\setminus\mathcal{Y}^{+}(x)}
|y|.
\label{eq:context}
\end{equation}
When all sampled rollouts are verified, we instead use another verified rollout and never use the scaffold itself as its own context. Figure~\ref{fig:budget}c reports the frequency of each case. Because the verifier is imperfect, a rollout in $\mathcal{Y}(x)\setminus\mathcal{Y}^{+}(x)$ is more precisely described as a trajectory that was not verified as correct rather than as a necessarily incorrect solution. The method therefore requires only a final answer per training problem, not a written solution.

\section{Experiments}
\label{sec:results}

\subsection{Setup}
\label{sec:exp-setup}

We follow the protocol of \citet{zhao2026selfdistilled} and change only the construction of $(r,y)$. We use Qwen3-1.7B, Qwen3-4B, and Qwen3-8B \citep{yang2025qwen3}. Training uses the mathematical subset of OpenThoughts \citep{guha2026openthoughts}. OPSD uses both the problem and its written solution, whereas OASIS uses only the final answer. All distillation methods use the same core settings unless otherwise stated, while Base, SFT, and GRPO use method-specific components. Complete distillation settings are provided in Table~\ref{tab:hparams}. For Table~\ref{tab:main}, we use a level playing field in terms of training samples. OPSD, AVSD, and CRISP use 32 problems per step for 100 steps, while OASIS samples $K=8$ rollouts for 64 candidate problems and trains on at most 32 verified problems. We compare the base model, supervised fine-tuning on the reference solutions, GRPO \citep{shao2024deepseekmath} with a binary outcome reward verified against the same answers, OPSD \citep{zhao2026selfdistilled}, AVSD \citep{nguyen2026avsd}, and CRISP~\citep{sang2026crisp}. Evaluation uses Avg@12 on AIME 2024, AIME 2025, and HMMT 2025 with the sampling settings recommended for Qwen3. Checkpoints are evaluated every 20 steps using the same checkpoint selection rule for all methods. Unless noted otherwise, each reported score is from a single training run evaluated once. Table~\ref{tab:qwen3-4b-opsd-oasis} additionally reports repeated-evaluation variability (three evaluation seeds on one Qwen3-4B checkpoint per method), confirming that this variation is small relative to the OASIS--OPSD gap. 

\begin{table}[t]
\centering
\caption{\textbf{Competition mathematics benchmark performance (Avg@12, \%).} Best result per block in \textbf{bold}, second best \underline{underlined}. $\Delta_{\text{Ours}-\text{OPSD}}$ highlights the growing margin of OASIS over OPSD as model scale increases.}
\label{tab:main}
\small
\resizebox{12cm}{!}{
\setlength{\tabcolsep}{10pt} % Optimal spacing between columns
\renewcommand{\arraystretch}{1.12} % Comfortable row height
\begin{tabular}{llcccc}
\toprule
\rowcolor{headergray}
\textbf{Model} & \textbf{Method} & \textbf{AIME24} & \textbf{AIME25} & \textbf{HMMT25} & \textbf{Average} \\
\midrule

\multirow{7}{*}{\textbf{Qwen3-1.7B}} 
& Base & 51.50 & 36.70 & 23.10 & 37.10 \\
& SFT & 48.40 & 36.70 & 22.70 & 35.93 \\
& GRPO & 51.10 & 38.30 & 23.70 & 37.70 \\
& OPSD & \underline{53.88} & \underline{40.80} & 25.76 & \underline{40.15} \\
& AVSD & 51.85 & 38.61 & \underline{26.39} & 38.95 \\
\rowcolor{ourshighlight}
& \textbf{OASIS (ours)} & \textbf{\color{bestcolor}54.17} & \textbf{\color{bestcolor}41.11} & \textbf{\color{bestcolor}26.94} & \textbf{\color{bestcolor}40.74} \\
\rowcolor{diffhighlight}
& \textbf{\textcolor{diffcolor}{$\Delta_{\text{Ours}-\text{OPSD}}$}} & \textbf{\textcolor{diffcolor}{+0.29}} & \textbf{\textcolor{diffcolor}{+0.31}} & \textbf{\textcolor{diffcolor}{+1.18}} & \textbf{\textcolor{diffcolor}{+0.59}} \\

\cmidrule(lr){1-6}

\multirow{8}{*}{\textbf{Qwen3-4B}} 
& Base & 73.02 & 65.00 & 41.56 & 59.86 \\
& SFT & 70.20 & 62.30 & 43.40 & 58.63 \\
& GRPO & 73.06 & 63.33 & \underline{45.00} & 60.46 \\
& OPSD & \underline{74.48} & 67.92 & 43.12 & \underline{61.84} \\
& AVSD & 73.06 & \underline{68.06} & 43.98 & 61.70 \\
& CRISP & 74.06 & 63.33 & 40.94 & 59.44 \\
\rowcolor{ourshighlight}
& \textbf{OASIS (ours)} & \textbf{\color{bestcolor}76.66} & \textbf{\color{bestcolor}69.17} & \textbf{\color{bestcolor}45.27} & \textbf{\color{bestcolor}63.70} \\
\rowcolor{diffhighlight}
& \textbf{\textcolor{diffcolor}{$\Delta_{\text{Ours}-\text{OPSD}}$}} & \textbf{\textcolor{diffcolor}{+2.18}} & \textbf{\textcolor{diffcolor}{+1.25}} & \textbf{\textcolor{diffcolor}{+2.15}} & \textbf{\textcolor{diffcolor}{+1.86}} \\

\cmidrule(lr){1-6}

\multirow{8}{*}{\textbf{Qwen3-8B}} 
& Base & 75.11 & 67.71 & 45.10 & 62.64 \\
& SFT & 73.61 & 66.67 & \underline{46.67} & 62.32 \\
& GRPO & 75.28 & \underline{71.67} & 41.67 & 62.87 \\
& OPSD & 73.75 & 68.75 & 45.84 & 62.78 \\
& AVSD & 75.74 & 70.09 & 44.91 & \underline{63.58} \\
& CRISP & \underline{76.04} & 67.29 & 42.92 & 62.08 \\
\rowcolor{ourshighlight}
& \textbf{OASIS (ours)} & \textbf{\color{bestcolor}77.50} & \textbf{\color{bestcolor}72.22} & \textbf{\color{bestcolor}47.77} & \textbf{\color{bestcolor}65.83} \\
\rowcolor{diffhighlight}
& \textbf{\textcolor{diffcolor}{$\Delta_{\text{Ours}-\text{OPSD}}$}} & \textbf{\textcolor{diffcolor}{+3.75}} & \textbf{\textcolor{diffcolor}{+3.47}} & \textbf{\textcolor{diffcolor}{+1.93}} & \textbf{\textcolor{diffcolor}{+3.05}} \\
\bottomrule
\end{tabular}}
\end{table}

\subsection{Main results}
\label{sec:main}

OASIS improves over OPSD by 0.59, 1.86, and 3.05 points on average at 1.7B, 4B, and 8B (Table~\ref{tab:main}). At 1.7B the difference is within the range that sampling variation can produce on 30 problems, so it alone is not conclusive. The differences at 4B and 8B are larger. GRPO uses the same verifier and answer labels but optimizes a scalar outcome reward, and improves over the base model by only 0.60, 0.60, and 0.23 points, against 3.64, 3.84, and 3.19 for OASIS, indicating that the dense teacher distribution contributes beyond the verifier alone. SFT on the written solutions is at or below the base model at all three scales. Across scale, OPSD's gain over the base model falls from 3.05 to 1.98 to 0.14 while the teacher's privileged advantage on failed prefixes falls from $+0.40$ to $+0.17$ (Table~\ref{tab:recovery}). This matches Section~\ref{sec:gap}: as the student grows stronger, the reference adds less information, whereas OASIS already restricts supervision to the regime where that component is small.

% \begin{table}[t]
\begin{wraptable}{r}{0.5\textwidth}
\vspace{-0.8cm}
\caption{Qwen3-4B reasoning benchmark results. Results are reported as mean $\pm$ standard deviation over 3 evaluation seeds with 12 samples per problem.}
\label{tab:qwen3-4b-opsd-oasis}
\resizebox{7cm}{!}{
\centering
\scalebox{0.8}{\begin{tabular}{llcc}
\toprule
\rowcolor{headergray}
\textbf{Benchmark} & \textbf{Metric} & \textbf{OPSD} & \textbf{OASIS} \\
\midrule
\multirow{3}{*}{AIME 2024}
& Avg@12  & $74.4 \pm 0.71$ & $\mathbf{76.6 \pm 0.23}$\\
& Pass@12 & $\mathbf{86.6 \pm 0.00}$ & $\mathbf{86.6 \pm 0.00}$\\
& Maj@12  & $80.00 \pm 0.00$ & $\mathbf{81.1 \pm 1.57}$\\
\midrule
\multirow{3}{*}{AIME 2025}
& Avg@12  & 67.9 $\pm 1.33$ & $\mathbf{69.1 \pm 0.23}$\\
& Pass@12 & 83.3 $\pm 0.00$ & $\mathbf{86.6 \pm 0.00}$\\
& Maj@12  & 73.3 $\pm 3.33$ & $\mathbf{80.0 \pm 0.00}$\\
\midrule
\multirow{3}{*}{HMMT 2025}
& Avg@12  & 43.1 $\pm 0.89$ & $\mathbf{45.2 \pm 0.23}$\\
& Pass@12 & $\mathbf{64.4 \pm 1.92}$ & 63.3 $\pm 0.00$\\ 
& Maj@12  & 51.11 $\pm 3.85$ & $\mathbf{56.6 \pm 2.72}$\\
\bottomrule
\end{tabular}}}
\end{wraptable}

\subsubsection{Results at scale}
\label{app:detailed-4b}
Table~\ref{tab:qwen3-4b-opsd-oasis} reports the performance of OPSD and OASIS with Qwen3-4B across three competition mathematics benchmarks. We report three complementary metrics. Avg@12 measures the mean accuracy across 12 sampled
solutions per problem. Pass@12 measures the fraction of problems for which at least one of the 12 samples is correct.  Each value is reported as the mean $\pm$ standard deviation over three evaluation seeds using the same trained checkpoint.

\paragraph{Training cost}
\label{sec:cost}

OASIS changes how many problems a model has to see, how many it is trained on, and what it needs for each of them. Table~\ref{tab:cost} compares the two sides of that
budget with the scores they buy. Table \ref{tab:cost} uses a different training configuration from Table \ref{tab:main}. In Table \ref{tab:cost}, OPSD and AVSD are trained with a batch size of 64 for up to 200 steps, whereas Table \ref{tab:main} uses a batch size of 32 for 100 steps. OASIS uses a batch size of 64 in both settings. However, because of its training construction, it effectively trains on no more than 32 samples per step. So, the baselines are trained on 12{,}800 problems, each with a written solution. OASIS sees half as many problems, 6{,}400, and trains on
roughly half of those, the ones for which at least one rollout verifies. It never uses a written solution, only a final answer per problem. With about a quarter of the supervised problems, OASIS matches or exceeds both baselines on average at both scales. At 4B it reaches 63.70 against 63.60 for OPSD and 62.90 for AVSD. At 8B the gap is larger, 65.83 against 64.80 and 64.03. The two columns where a baseline is ahead, HMMT 2025 at 4B and AIME 2024 at 8B, are within about one point. The price is paid in generation. OASIS samples eight rollouts for every problem it sees, so it generates 51{,}200 rollouts against 12{,}800 for a one-rollout OPSD run on the larger problem set. Sampling is the cheaper side of the trade in most settings, since it needs no labeled data and parallelizes well, while written solutions do not exist for most datasets with checkable answers. Where generation is the binding constraint, the $K=4$ ablation in Appendix~\ref{app:ablations} halves this cost. 
Algorithm~\ref{alg:oasis} (Appendix~\ref{app}) summarises one OASIS update. With $K=8$, OASIS generates 43.5M tokens over 100 steps at 1.7B versus 2.76M for OPSD, but supervises fewer, 1.98M versus 2.76M, because about half of the problems are masked and verified scaffolds are shorter (603 versus 859 tokens, Table~\ref{tab:budget}).

\begin{table}[t]
\centering
\caption{\textbf{ Written-supervision and rollout trade-offs.}This table reports annotation, rollout, and supervised-example trade-offs. Problems seen and problems that received a gradient over training, rollouts generated, and whether written solutions are required. Scores are Avg@12. Best per column within each model in \textbf{bold}.}
\label{tab:cost}
\small
\resizebox{13cm}{!}{
\setlength{\tabcolsep}{4.5pt}
\begin{tabular}{llcccc>{\raggedright\arraybackslash}p{2.4cm}cccc}
\toprule
\rowcolor{headergray}
& & \multicolumn{4}{c}{\textbf{Training budget}} & \multicolumn{4}{c}{\textbf{Accuracy}} \\
\cmidrule(lr){3-6}\cmidrule(lr){7-10}
\rowcolor{headergray}
\textbf{Model} & \textbf{Method}
& \textbf{Seen} & \textbf{Trained} & \textbf{Rollouts} & \textbf{Written solution?}
& \textbf{AIME24} & \textbf{AIME25} & \textbf{HMMT25} & \textbf{Avg.} \\
\midrule
\multirow{3}{*}{Qwen3-4B}
& OPSD$^{\dagger}$ & 12.8K & 12.8K & 12.8K & Yes & 76.40 & 68.30 & \textbf{46.10} & 63.60 \\
& AVSD$^{\dagger}$ & 12.8K & 12.8K & 12.8K & Yes & 76.20 & 68.30 & 44.20 & 62.90 \\
\rowcolor{ourshighlight}
& \textbf{OASIS} & \textbf{6.4K} & \textbf{$\approx$3.2K} & 51.2K & \textbf{No}
& \textbf{76.66} & \textbf{69.17} & 45.27 & \textbf{63.70} \\
\midrule
\multirow{3}{*}{Qwen3-8B}
& OPSD$^{\dagger}$ & 12.8K & 12.8K & 12.8K & Yes & \textbf{77.80} & 70.80 & 45.80 & 64.80 \\
& AVSD$^{\dagger}$ & 12.8K & 12.8K & 12.8K & Yes & 75.40 & 69.60 & 47.10 & 64.03 \\
\rowcolor{ourshighlight}
& \textbf{OASIS} & \textbf{6.4K} & \textbf \textbf{$\approx$3.2K} & 51.2K & \textbf{No} 
& 77.50 & \textbf{72.22} & \textbf{47.77} & \textbf{65.83} \\
\bottomrule
\end{tabular}}
\end{table}

\subsection{Training behaviour}
\label{sec:training-behaviour}

At 1.7B, OASIS keeps a higher fraction of verified rollouts than OPSD (0.314 versus 0.283 on average), its student entropy reaches 0.49 versus 0.70, and its top-1 agreement with the teacher stays at 0.931 while OPSD's falls from 0.932 to 0.888 (Figure~\ref{fig:dynamics}), consistent with Proposition~\ref{prop:mixture}. At 4B the separation lasts for about two thirds of training, after which OASIS also drifts: top-1 agreement ends at 0.853 and about half of the tokens have a clipped entry (Table~\ref{tab:budget}). Clipping is one possible explanation (Appendix~\ref{app:clip}), so we evaluate both methods under the same checkpoint selection rule.

\begin{takeawaybox}
Even when the teacher is conditioned on the student’s own unverified rollout, while the student is supervised on a verified rollout selected primarily by its final output label, distillation can outperform standard OPSD.
\end{takeawaybox}

\section{Conclusion}
\label{sec:conclusion}

We find that OPSD's privileged teacher is most useful precisely where its signal is least imitable. The teacher has a large advantage on failed trajectories, where OPSD concentrates most of its supervision, but little advantage on verified trajectories. OASIS therefore restricts distillation to verified on-policy scaffolds and replaces the written solution with a distinct same-problem rollout, typically the student's own unverified attempt, retaining the teacher's dense signal without requiring reference solutions. Across Qwen3-1.7B, 4B and 8B, OASIS improves over the base model by 3.64, 3.84 and 3.19 points, while OPSD's gain falls from 3.05 to 1.98 and 0.14. These results suggest that the effectiveness of privileged self-distillation depends on whether the teacher's supervision can be transferred to the student's available context.
Our study is limited to competition mathematics and one model family, with higher generation cost and single runs at 4B and 8B. Extending this criterion to settings with inexpensive outcome verification but no written solutions, such as code with unit tests, is a natural next step.

\section*{Acknowledgement}
This work was supported in part by a generous compute grant from Lambda AI. We gratefully acknowledge their support.

\bibliography{iclr2025_conference}

@inproceedings{
zhao2026selfdistilled,
title={Self-Distilled Reasoner: On-Policy Self-Distillation for Large Language Models},
author={Siyan Zhao and Zhihui Xie and Mengchen Liu and Jing Huang and Guan Pang and Feiyu Chen and Aditya Grover},
booktitle={Forty-third International Conference on Machine Learning},
year={2026},
url={https://openreview.net/forum?id=Jpxfof0EaS}
}

@article{hinton2015distilling,
  title={Distilling the knowledge in a neural network},
  author={Hinton, Geoffrey and Vinyals, Oriol and Dean, Jeff},
  journal={arXiv preprint arXiv:1503.02531},
  year={2015}
}

@article{romero2015fitnets,
  title={Fitnets: Hints for thin deep nets},
  author={Romero, Adriana and Ballas, Nicolas and Kahou, Samira Ebrahimi and Chassang, Antoine and Gatta, Carlo and Bengio, Yoshua},
  journal={arXiv preprint arXiv:1412.6550},
  year={2014}
}

@inproceedings{kim2016sequence,
  title={Sequence-level knowledge distillation},
  author={Kim, Yoon and Rush, Alexander M},
  booktitle={Proceedings of the 2016 conference on empirical methods in natural language processing},
  pages={1317--1327},
  year={2016}
}

@inproceedings{gu2024minillm,
  title={Minillm: Knowledge distillation of large language models},
  author={Gu, Yuxian and Dong, Li and Wei, Furu and Huang, Minlie},
  booktitle={International Conference on Learning Representations},
  volume={2024},
  pages={32694--32717},
  year={2024}
}

@article{shen2026purified,
  title={Purified OPSD: On-Policy Self-Distillation Without Losing How to Think},
  author={Shen, Zhanming and Tong, Jintao and Yan, Shaotian and Shen, Chen and Chen, Hao and Ye, Wentao and Hu, Xiaomeng and Miao, Rui and Wang, Haobo and Zhao, Junbo and others},
  journal={arXiv preprint arXiv:2607.02234},
  year={2026}
}

@article{ichihara2026privileged,
  title={Privileged Solutions or Context-Induced Teacher Behavior? Dissecting On-Policy Self-Distillation},
  author={Ichihara, Yuki and Iwase, Naoto and Quamar, Mohammad Atif and Komiyama, Junpei},
  journal={arXiv preprint arXiv:2608.09228},
  year={2026}
}

@article{uesato2022solving,
  title={Solving math word problems with process-and outcome-based feedback},
  author={Uesato, Jonathan and Kushman, Nate and Kumar, Ramana and Song, Francis and Siegel, Noah and Wang, Lisa and Creswell, Antonia and Irving, Geoffrey and Higgins, Irina},
  journal={arXiv preprint arXiv:2211.14275},
  year={2022}
}

@inproceedings{lightman2024let,
  title={Let's verify step by step},
  author={Lightman, Hunter and Kosaraju, Vineet and Burda, Yuri and Edwards, Harrison and Baker, Bowen and Lee, Teddy and Leike, Jan and Schulman, John and Sutskever, Ilya and Cobbe, Karl},
  booktitle={International Conference on Learning Representations},
  volume={2024},
  pages={39578--39601},
  year={2024}
}

@inproceedings{ross2011reduction,
  title={A reduction of imitation learning and structured prediction to no-regret online learning},
  author={Ross, St{\'e}phane and Gordon, Geoffrey and Bagnell, Drew},
  booktitle={Proceedings of the fourteenth international conference on artificial intelligence and statistics},
  pages={627--635},
  year={2011},
  organization={JMLR Workshop and Conference Proceedings}
}

@article{shao2024deepseekmath,
  title={Deepseekmath: Pushing the limits of mathematical reasoning in open language models},
  author={Shao, Zhihong and Wang, Peiyi and Zhu, Qihao and Xu, Runxin and Song, Junxiao and Bi, Xiao and Zhang, Haowei and Zhang, Mingchuan and Li, YK and Wu, Yang and others},
  journal={arXiv preprint arXiv:2402.03300},
  year={2024}
}

@inproceedings{
wang2022self,
title={Self-Consistency Improves Chain of Thought Reasoning in Language Models},
author={Xuezhi Wang and Jason Wei and Dale Schuurmans and Quoc V Le and Ed H. Chi and Sharan Narang and Aakanksha Chowdhery and Denny Zhou},
booktitle={The Eleventh International Conference on Learning Representations },
year={2023},
url={https://openreview.net/forum?id=1PL1NIMMrw}
}

@article{zelikman2022star,
  title={Star: Bootstrapping reasoning with reasoning},
  author={Zelikman, Eric and Wu, Yuhuai and Mu, Jesse and Goodman, Noah},
  journal={Advances in Neural Information Processing Systems},
  volume={35},
  pages={15476--15488},
  year={2022}
}

@article{swamy2022sequence,
  title={Sequence model imitation learning with unobserved contexts},
  author={Swamy, Gokul and Choudhury, Sanjiban and Bagnell, J and Wu, Steven},
  journal={Advances in Neural Information Processing Systems},
  volume={35},
  pages={17665--17676},
  year={2022}
}

@article{yuan2023scaling,
  title={Scaling relationship on learning mathematical reasoning with large language models},
  author={Yuan, Zheng and Yuan, Hongyi and Li, Chengpeng and Dong, Guanting and Lu, Keming and Tan, Chuanqi and Zhou, Chang and Zhou, Jingren},
  journal={arXiv preprint arXiv:2308.01825},
  year={2023}
}

@article{yang2025qwen3,
  title={Qwen3 technical report},
  author={Yang, An and Li, Anfeng and Yang, Baosong and Zhang, Beichen and Hui, Binyuan and Zheng, Bo and Yu, Bowen and Gao, Chang and Huang, Chengen and Lv, Chenxu and others},
  journal={arXiv preprint arXiv:2505.09388},
  year={2025}
}

@inproceedings{guha2026openthoughts,
  title={Openthoughts: Data recipes for reasoning models},
  author={Guha, Etash and Marten, Ryan and Keh, Sedrick and Raoof, Negin and Smyrnis, Georgios and Bansal, Hritik and Nezhurina, Marianna and Mercat, Jean and Vu, Trung and Sprague, Zayne and others},
  booktitle={International Conference on Learning Representations},
  volume={2026},
  pages={108059--108130},
  year={2026}
}

@inproceedings{agarwal2024policy,
  title={On-policy distillation of language models: Learning from self-generated mistakes},
  author={Agarwal, Rishabh and Vieillard, Nino and Zhou, Yongchao and Stanczyk, Piotr and Ramos Garea, Sabela and Geist, Matthieu and Bachem, Olivier},
  booktitle={International Conference on Learning Representations},
  volume={2024},
  pages={21246--21263},
  year={2024}
}

@article{vapnik2009lupi,
  title={On the theory of learnining with privileged information},
  author={Pechyony, Dmitry and Vapnik, Vladimir},
  journal={Advances in neural information processing systems},
  volume={23},
  year={2010}
}

@article{weihs2021bridging,
  title={Bridging the imitation gap by adaptive insubordination},
  author={Weihs, Luca and Jain, Unnat and Liu, Iou-Jen and Salvador, Jordi and Lazebnik, Svetlana and Kembhavi, Aniruddha and Schwing, Alex},
  journal={Advances in Neural Information Processing Systems},
  volume={34},
  pages={19134--19146},
  year={2021}
}

@InProceedings{pinto2018asymmetric,
  title = 	 {A Theoretical Justification for Asymmetric Actor-Critic Algorithms},
  author =       {Lambrechts, Gaspard and Ernst, Damien and Mahajan, Aditya},
  booktitle = 	 {Proceedings of the 42nd International Conference on Machine Learning},
  pages = 	 {32375--32405},
  year = 	 {2025},
  editor = 	 {Singh, Aarti and Fazel, Maryam and Hsu, Daniel and Lacoste-Julien, Simon and Berkenkamp, Felix and Maharaj, Tegan and Wagstaff, Kiri and Zhu, Jerry},
  volume = 	 {267},
  series = 	 {Proceedings of Machine Learning Research},
  month = 	 {13--19 Jul},
  publisher =    {PMLR},
  url = 	 {https://proceedings.mlr.press/v267/lambrechts25a.html} 
}

@article{anthony2017expert,
  title={Thinking fast and slow with deep learning and tree search},
  author={Anthony, Thomas and Tian, Zheng and Barber, David},
  journal={Advances in neural information processing systems},
  volume={30},
  year={2017}
}

@article{schulman2017ppo,
  title={Proximal policy optimization algorithms},
  author={Schulman, John and Wolski, Filip and Dhariwal, Prafulla and Radford, Alec and Klimov, Oleg},
  journal={arXiv preprint arXiv:1707.06347},
  year={2017}
}

@article{nguyen2026avsd,
  title={AVSD: Adaptive-View Self-Distillation by Balancing Consensus and Teacher-Specific Privileged Signals},
  author={Nguyen, Duy and Xiao, Hanqi and Prasad, Archiki and Khan, Zaid and Das, Anirban and Zhang, Austin and Sahu, Sambit and Lee, Hyunji and Stengel-Eskin, Elias and Bansal, Mohit},
  journal={arXiv preprint arXiv:2605.20643},
  year={2026}
}

@article{sang2026crisp,
  title={Crisp: Compressed reasoning via iterative self-policy distillation},
  author={Sang, Hejian and Xu, Yuanda and Zhou, Zhengze and He, Ran and Wang, Zhipeng and Sun, Jiachen},
  journal={arXiv preprint arXiv:2603.05433},
  year={2026}
}

@article{ke2026respecting,
  title={Respecting self-uncertainty in on-policy self-distillation for efficient llm reasoning},
  author={Ke, Junlong and Wen, Zichen and Li, Weijia and He, Conghui and Zhang, Linfeng},
  journal={arXiv preprint arXiv:2605.13255},
  year={2026}
}

@article{hou2026dash,
  title={Dash: Divergence-adaptive supervision horizons for on-policy self-distillation of reasoning models},
  author={Hou, ZhiYan and Tang, Xinyu and An, Hongyan and Zhang, Jianjin and Wang, Weizhen and Han, Yunyun and Li, Gengsheng and Hao, Xiangzhao and Guo, Haiyun and Hu, Wenbin and others},
  journal={arXiv preprint arXiv:2608.06243},
  year={2026}
}
\bibliographystyle{iclr2027_conference}

\appendix
\section{Appendix}

\subsection{Proofs}
\label{app:proofs}

\subsubsection{Proposition~\ref{prop:mixture}: the forward-KL optimum under privileged context}

\begin{proposition*}
Let $Z$ denote the information available to the student at a supervised state, here $z=(x,y_{<t})$, and let $R$ denote the privileged context. For a fixed $z$, let $R\sim P(r\mid z)$ and let $q(\cdot\mid z,r)$ be the teacher's next-token distribution. If the student distribution $p(\cdot\mid z)$ is restricted to depend only on $z$, then the minimizer of the expected forward KL divergence

$$
\mathbb{E}_{R\mid z}
\left[
\mathrm{KL}\!\left(
q(\cdot\mid z,R)\,\|\,p(\cdot\mid z)
\right)
\right]
$$

is

$$
p^{\star}(\cdot\mid z)
=
\mathbb{E}_{R\mid z}
\left[q(\cdot\mid z,R)\right].
$$

The minimum value is the conditional mutual information
$I(A;R\mid Z=z)$, where $A$ is sampled from
$q(\cdot\mid z,R)$. Moreover,

$$
H\!\left(p^{\star}(\cdot\mid z)\right)
\geq
\mathbb{E}_{R\mid z}
\left[
H\!\left(q(\cdot\mid z,R)\right)
\right].
$$

\end{proposition*}

\begin{proof}
For a fixed $z$, define the mixture

$$
\bar q(a)
=
\mathbb{E}_{R\mid z}
\left[q(a\mid z,R)\right].
$$

Expanding the expected forward KL gives

$$
\begin{aligned}
\mathbb{E}_{R\mid z}
\left[
\mathrm{KL}(q\,\|\,p)
\right]
&=
\mathbb{E}_{R\mid z}
\left[
\sum_a q(a\mid z,R)\log q(a\mid z,R)
\right]
-
\sum_a \bar q(a)\log p(a\mid z).
\end{aligned}
$$

The first term does not depend on $p$. For the second term,

$$
-\sum_a \bar q(a)\log p(a\mid z)
=
-\sum_a \bar q(a)\log \bar q(a)
+
\mathrm{KL}\!\left(\bar q\,\|\,p\right).
$$

Since the KL divergence is non-negative and is zero only when $p=\bar q$, the unique minimizer is

$$
p^{\star}(\cdot\mid z)=\bar q(\cdot).
$$

Substituting this solution into the objective gives

$$
\mathbb{E}_{R\mid z}
\left[
\mathrm{KL}\!\left(
q(\cdot\mid z,R)\,\|\,\bar q
\right)
\right].
$$

This is precisely the conditional mutual information
$I(A;R\mid Z=z)$ under
$A\sim q(\cdot\mid z,R)$ and $R\sim P(\cdot\mid z)$.

Finally, entropy is concave, so Jensen's inequality gives

$$
H\!\left(
\mathbb{E}_{R\mid z}[q(\cdot\mid z,R)]
\right)
\geq
\mathbb{E}_{R\mid z}
\left[
H(q(\cdot\mid z,R))
\right],
$$

which proves the entropy statement.
\end{proof}

\paragraph{Interpretation and scope.}
The proposition gives the behaviour of the \emph{unclipped} forward-KL objective when the student cannot condition on the privileged context. The student cannot in general reproduce a teacher whose distribution changes with $R$. Instead, the optimal student matches the average teacher distribution across contexts compatible with its observations. The residual loss is therefore $I(A;R\mid Z)$, which is zero only when the teacher's next-token distribution contains no additional information about $R$ given $Z$.

The entropy inequality gives a related consequence. The mixture that the student can represent is at least as entropic as the corresponding teacher distributions on average. This provides a simple explanation for why supervising a student with a teacher that depends on inaccessible information can produce a less concentrated student.

The proposition should not be read as an exact characterization of the implemented OPSD objective, since OPSD applies the per-entry clipping defined in Eq.~\ref{eq:opsd}. Instead, it isolates the privileged-information effect under the underlying forward-KL objective. The empirical recovery gap $\Delta$ in Section~\ref{sec:advantage} provides an outcome-level measurement of the same phenomenon. %

There is also a subtlety concerning the distribution $P(r\mid z)$. In our setting, the reference solution is associated deterministically with the problem, so a sufficiently capable student could, in principle, infer information about $r$ from $x$. The proposition concerns the regime in which the privileged context contains information that is not recovered by the student's observable state. This is precisely the regime probed empirically. At failed prefixes, the gold-context teacher recovers the correct answer substantially more often than the student without the reference (Table~\ref{tab:recovery}).

\subsubsection{Proposition~\ref{prop:clip}: the per-entry clip can reverse an update}
\label{app:clip-proof}

\begin{proposition*}
Let $p=\mathrm{softmax}(z)$, let

$$
e_a=q_a\log\frac{q_a}{p_a},
$$

and define

$$
U=\{a:e_a\leq\kappa\}.
$$

For the clipped objective

$$
\ell(z)=\sum_a\min(e_a,\kappa),
$$

and away from the clipping boundary, the gradient with respect to the student
logit $z_j$ is

$$
\frac{\partial\ell}{\partial z_j}
=
p_j\sum_{a\in U}q_a
-
q_j\mathbf{1}[j\in U].
$$

In particular, if $j\notin U$, then

$$
\frac{\partial\ell}{\partial z_j}
=
p_j\sum_{a\in U}q_a>0,
$$
so a gradient-descent step decreases the logit $z_j$. Because every logit is updated simultaneously, this does not by itself imply that $p_j$ decreases.
\end{proposition*}

\begin{proof}
Away from the clipping boundary, an entry contributes to the gradient only
when $e_a\leq\kappa$. For such an entry,

$$
e_a
=
q_a\log q_a-q_a\log p_a.
$$

Since

$$
\frac{\partial\log p_a}{\partial z_j}
=
\mathbf{1}[a=j]-p_j,
$$

we obtain

$$
\frac{\partial e_a}{\partial z_j}
=
-q_a
\left(
\mathbf{1}[a=j]-p_j
\right).
$$

Summing over the unclipped entries therefore gives

$$
\begin{aligned}
\frac{\partial\ell}{\partial z_j}
&=
\sum_{a\in U}
-q_a
\left(
\mathbf{1}[a=j]-p_j
\right)\\
&=
-q_j\mathbf{1}[j\in U]
+
p_j\sum_{a\in U}q_a.
\end{aligned}
$$

If $j\notin U$, the first term is absent and

$$
\frac{\partial\ell}{\partial z_j}
=
p_j\sum_{a\in U}q_a>0.
$$

Gradient descent consequently decreases $z_j$. The probability $p_j$ need not decrease, because the other logits change in the same step. For example, with $p=(0.8,0.1,0.1)$, $q=(0.1,0.1,0.8)$ and $\kappa=0.05$, only the third entry is clipped, and a small gradient step lowers $z_3$ while raising $p_3$.
\end{proof}

\paragraph{Interpretation.}
The clipping effect is different from the privileged-information effect in Proposition~\ref{prop:mixture}. It arises from the form of the implemented objective itself. Since

$$
e_a>\kappa>0
\quad\Longrightarrow\quad
q_a>p_a,
$$

a clipped coordinate is one for which the teacher assigns more probability than the student. Without clipping, the usual forward-KL gradient would apply upward pressure to such a token. Once that coordinate is clipped, its direct
gradient contribution is removed, while the normalization term from the remaining coordinates remains. The resulting gradient on that logit therefore points in the opposite direction.

This provides a possible explanation for the increasing drift from the teacher observed during training. As the student moves away from the teacher, more entries satisfy the clipping condition, which can in turn create additional
drift. We therefore analyse clipping separately from the privileged-information gap and keep the same clipping rule for OPSD and OASIS throughout the experiments.

The derivation was checked numerically using central finite differences on random $(z,q)$ pairs, with agreement within $10^{-5}$. Figure~\ref{fig:clipping} quantifies the prevalence of the clipping effect in our experiments.

\begin{figure}[t]
    \centering
    \includegraphics[width=0.85\textwidth]{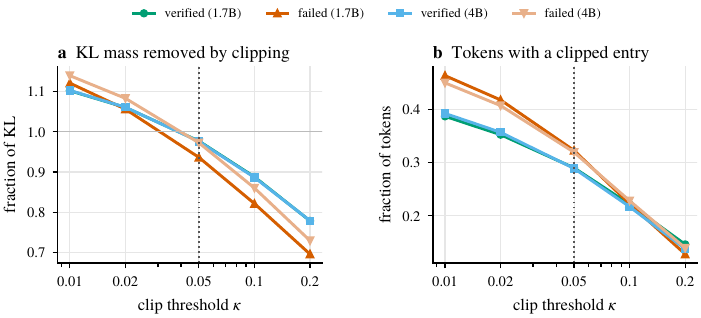}
    \caption{
    \textbf{Effect of the per-entry clip.}
    \textbf{(a)} Fraction of the positive KL mass removed by the clip. Values above one indicate that the clipped terms exceed the total positive KL mass, resulting in a negative value of the clipped objective.
    \textbf{(b)} Fraction of tokens with at least one clipped vocabulary entry. The dotted line marks $\kappa=0.05$, used by OPSD and all experiments in this work. Proposition~\ref{prop:clip} shows that clipped entries lose the teacher's upward pull and receive a downward logit gradient.
    }
    \label{fig:clipping}
\end{figure}

\begin{table}[t]
\centering
\caption{\textbf{Training composition and dynamics.} Statistics over 100 optimizer steps with 64 problems per step and $K=8$ rollouts per problem for OASIS. ``Problems receiving gradient'' is the fraction of candidate problems with at least one verified rollout. ``Selected scaffolds reaching the answer'' is computed only among the scaffolds selected for supervision, while ``Candidate rollouts hitting the length cap'' is computed over all generated candidate rollouts. OASIS supervises fewer tokens because problems without a verified rollout are excluded and selected scaffolds are shorter on average, while its additional cost comes from rollout generation. Entropy, teacher--student top-1 agreement, and clipped-token fraction are measured on training batches, with early and late values computed over steps 2--34 and 68--100, respectively.}
\label{tab:budget}
\small
\resizebox{14cm}{!}{
\begin{tabular}{lcccc}
\toprule
& \multicolumn{2}{c}{\textbf{Qwen3-1.7B}} & \multicolumn{2}{c}{\textbf{Qwen3-4B}} \\
\cmidrule(lr){2-3}\cmidrule(lr){4-5}
& \textbf{OPSD} & \textbf{OASIS} & \textbf{OPSD} & \textbf{OASIS} \\
\midrule
Problems receiving gradient               & 100\% & 51.7\% & 100\% & 50.8\% \\
Selected scaffolds reaching the answer    & 28.3\% & 100\% & 30.6\% & 100\% \\
Candidate rollouts hitting the length cap & 62.0\% & 58.0\% & 64.4\% & 63.1\% \\
Mean supervised length (tokens)           & 859 & 603 & 883 & 646 \\
Generated tokens                          & 2.76M & 43.5M & 2.80M & 44.8M \\
Supervised tokens                         & 2.76M & 1.98M & 2.80M & 2.07M \\
\midrule
Student entropy (early $\to$ late)         & 0.40 $\to$ 0.70 & 0.25 $\to$ 0.49 & 0.35 $\to$ 0.69 & 0.28 $\to$ 0.67 \\
Teacher--student top-1 (early $\to$ late) & 0.932 $\to$ 0.888 & 0.929 $\to$ 0.931 & 0.938 $\to$ 0.882 & 0.926 $\to$ 0.853 \\
Clipped tokens (early $\to$ late)          & 0.31 $\to$ 0.44 & 0.25 $\to$ 0.35 & 0.28 $\to$ 0.44 & 0.29 $\to$ 0.51 \\
\bottomrule
\end{tabular}}
\end{table}

\begin{table}[t]
\centering
\caption{\textbf{Teacher signal at verified and failed states.} Measurements use the gold reference on 200 rollouts per condition for Qwen3-1.7B. ``Sampled'' denotes the token actually produced by the student at the supervised state.}
\label{tab:states}
\small
\begin{tabular}{lcc}
\toprule
\rowcolor{headergray}
\textbf{Quantity at the supervised state}
& \textbf{Verified rollouts}
& \textbf{Failed rollouts} \\
\midrule
Share of OPSD's training scaffolds        & 28.3\% & 71.7\% \\
KL$(q_t \,\|\, p_t)$ per token            & 0.274 & 0.181 \\
Teacher entropy                           & 0.297 & 0.461 \\
Teacher probability on the sampled token  & 0.834 & 0.775 \\
Teacher top-1 $=$ sampled token           & 0.862 & 0.814 \\
Tokens with a clipped entry ($\kappa=0.05$) & 0.288 & 0.323 \\
\bottomrule
\end{tabular}
\end{table}

\subsection{Probe protocols}
\label{app:probes}

All probes are inference-only and use the frozen initial policy $\pi_{\theta_0}$, which is also the teacher used during training. Scaffolds are taken from the generation logs of the corresponding training runs. Unless otherwise stated, we use 200 scaffolds per cell, sampled without replacement with at most one scaffold per problem. We restrict the probes to scaffolds generated during the first 40 optimizer steps, keeping the generating policy
close to the initial policy used by the frozen teacher.

\paragraph{Prompts and contexts.}
Student and teacher prompts are reconstructed exactly as in training, including the thinking-mode setting. We evaluate the following reference context conditions: 
(1)~\textbf{Gold}: the dataset's written solution. 
(2)~\textbf{Own failed rollout}: the failed rollout associated with the same training example and used as the OASIS context. 
(3)~\textbf{Own correct rollout}: a verified rollout from the same problem, used when all sampled rollouts are correct. 
(4)~\textbf{Unrelated}: the written solution of a different problem. 
(5)~\textbf{Shuffled gold}: the gold solution with its lines randomly permuted. 
(6)~\textbf{Empty}: the reference block and its instruction with no reference content. 
(7)~\textbf{Template only}: the teacher prompt with the thinking-mode setting but without a reference block. 
(8)~\textbf{Answer only}: the string \texttt{The final answer is \textbackslash boxed\{a\}}. and 
(9)~\textbf{Gold, answer removed}: the gold solution with boxed answer payloads and the final paragraph removed when it still states the answer.\footnote{Because the answer may occur elsewhere in a derivation, we also report the subset in which the answer string does not reappear (marked with $\star$ in Table~\ref{tab:recovery}).}
\paragraph{Signal probe.}
For each scaffold, we evaluate the student distribution $p_t$ and the teacher distribution $q_t$ at every scaffold position for each context, using the distillation temperature $T=1.1$. We measure the total-variation distance from the gold-context teacher, top-1 agreement, the teacher--student $\mathrm{KL}(q_t|p_t)$ before and after clipping, teacher entropy, and the teacher probability assigned to the token sampled by the student. For position-dependent analyses, quantities are aggregated both by relative position within a scaffold and by absolute token position. Reported means are computed per scaffold, with 95\% bootstrap intervals over scaffolds.

The decomposition in Figure~\ref{fig:signal} is a sequential attribution of the teacher--student KL gap. Starting from the student prompt, we successively add the thinking-mode setting, the reference block and its instruction, and finally the reference content. Each increment is reported as a fraction of the full gap. Because sequential attribution depends on the order of the additions, we also computed the corresponding decomposition using total variation. It gives the same qualitative ordering.

\paragraph{Recovery probe.}
For each scaffold, we retain the first
$f\in\{0.1,0.3,0.5,0.7,0.9\}$ fraction of its tokens and use the resulting prefix as the starting state for continuation. For each context, we sample four continuations at temperature $1.0$, top-$p=0.95$, top-$k=20$, with a maximum of 2048 new tokens.

A continuation is counted as a recovery if the last boxed answer in the prefix and continuation matches the ground-truth answer. We additionally record whether the continuation reproduces the scaffold's original wrong answer and whether it produces no answer before reaching the token limit.

Recovery rates are first averaged over continuations for each problem and then over problems, with 95\% bootstrap intervals. We exclude $f=0$ because a thinking-mode continuation without a prefix frequently reaches the token limit before producing an answer, making it poorly comparable with the other cut fractions.

Failed scaffolds include both trajectories ending in an incorrect answer and trajectories that reach the token limit without producing an answer. An additional analysis restricted to scaffolds that end with an incorrect answer is reported in Appendix~\ref{app:ablations}.

%\paragraph{Probe limitations.}
%The probes have three limitations. First, scaffolds are reconstructed from decoded text and re-tokenised, which can differ from the original sampled token sequence in rare cases. Second, the probes use the frozen initial policy as the teacher while the scaffolds are generated by a partially trained student. We therefore restrict the main analysis to early training steps. Third, the recovery gap $\Delta$ measures the difference in outcomes when the teacher or student continues from the same prefix, whereas the training objective evaluates teacher distributions at fixed scaffold states. Thus, $\Delta$ is anoutcome-level proxy for the privileged information available to the teacher, rather than a direct estimate of the token-level quantity $I(A;R\mid Z)$.

\begin{figure}[t]
    \centering
    \includegraphics[width=0.85\textwidth]{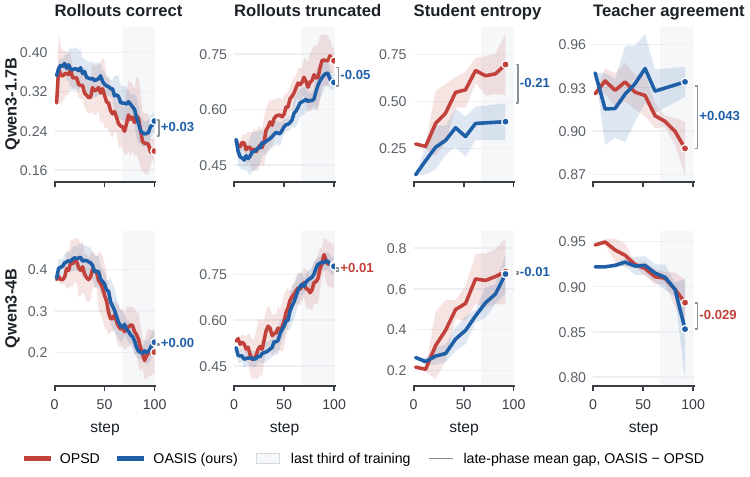}
    \caption{\textbf{Training dynamics at two scales (rows) for four metrics (columns).} Lines represent trailing means with $\pm 1$ standard deviation bands. Brackets report the OASIS minus OPSD gap averaged over the shaded final third of training. At 1.7B, OASIS maintains a higher correct rate, less truncation, lower student entropy, and stable teacher agreement, whereas OPSD drifts. At 4B, both methods track each other for most of training and drift near the end, which aligns with the clipping effect described in Proposition~\ref{prop:clip}. Metrics are computed on each method’s own training batches. Differences therefore reflect both changes in model behavior and changes in the selected trajectory distribution.}
    \label{fig:dynamics}
\end{figure}

\subsection{Clipping}
\label{app:clip}

\begin{proposition}[Clipping Removes the Teacher's Pull]
\label{prop:clip}
Let $p, q \in \Delta^{|V|-1}$ denote the student and teacher probability distributions over vocabulary $V$, and let $e_a = q_a \log(q_a / p_a)$ represent the per-entry KL contribution. Define $U = \{a \in V : e_a \leq \kappa\}$ as the set of unclipped tokens. For the clipped OPSD objective $\mathcal{L} = \sum_{a \in V} \min(e_a, \kappa)$, the gradient with respect to the student logit $z_j$ is:
\begin{equation}
g_j = \frac{\partial \mathcal{L}}{\partial z_j} = p_j \sum_{a \in U} q_a - q_j \cdot \mathbf{1}[j \in U].
\end{equation}
Consequently, for any clipped token $j \notin U$, the gradient simplifies to $g_j = p_j \sum_{a \in U} q_a > 0$. Since $e_j > \kappa > 0$ implies $q_j > p_j$, the unclipped forward-KL gradient $p_j - q_j$ would be negative and raise $z_j$. Under clipping the sign is reversed, so a gradient-descent step ($\Delta z_j = -\eta g_j$) lowers $z_j$, regardless of how strongly the teacher preferred token $j$.
\end{proposition}

The proof is given in Appendix~\ref{app:clip-proof}. Clipping therefore does not merely freeze a token's update. Whenever an entry's KL contribution exceeds $\kappa$, the teacher's upward pull on that token is removed and the sign of its logit gradient is reversed, so the student receives no signal toward tokens the teacher prefers most strongly. Because all logits move together, the token's probability need not fall, but it is no longer pushed toward the teacher.

Figure~\ref{fig:clipping} shows the fraction of
positive KL mass removed by the per-entry clip as a function of $\kappa$, measured on verified and failed scaffolds. At the value used by OPSD and throughout our
experiments, $\kappa=0.05$, the clip removes 97\% of the positive mass on verified scaffolds and 94\% on failed scaffolds. Approximately three tokens in ten have at least one clipped entry. Because the clipped sum can become negative, the reported training loss is negative from the beginning of training.

Together with Proposition~\ref{prop:clip}, these results provide a possible explanation for the late-training drift observed in both methods. Clipped entries lose the teacher's upward pull, which can move the student further from the teacher and causes more entries to cross the clipping threshold. The clipped fraction increases from 0.31 to 0.44 for OPSD and from 0.25 to 0.35 for OASIS at 1.7B over 100 steps.  The analysis nevertheless suggests that a larger threshold, or clipping the per-token divergence rather than individual vocabulary entries, may reduce this effect.

\begin{figure}[t]
    \centering
    \includegraphics[width=0.85\textwidth]{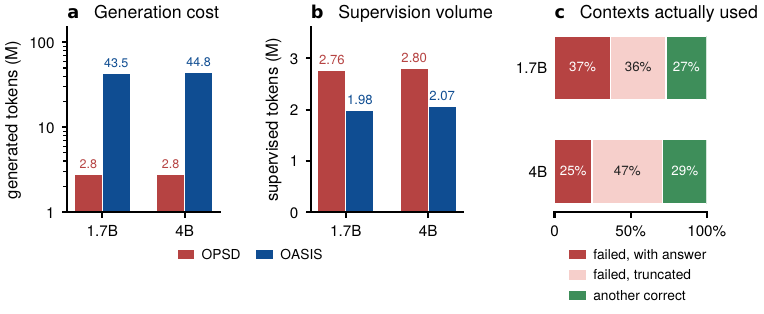}
    \caption{
    \textbf{Token budget over 100 optimizer steps.}
    \textbf{(a)} Total tokens generated by OPSD and OASIS, shown on a logarithmic scale. OASIS generates roughly sixteen times more tokens because it samples $K=8$ rollouts per problem. \textbf{(b)} Tokens receiving distillation supervision. OASIS supervises about 30\% fewer tokens because problems without a verified rollout are excluded and verified scaffolds are shorter on average. \textbf{(c)} Teacher contexts used by OASIS: a failed rollout with a wrong answer, a failed rollout truncated at the token limit, or another verified rollout when every sampled rollout is correct.
    }
    \label{fig:budget}
\end{figure}

\subsection{Ablations}
\label{app:ablations}

Both ablations evaluate Qwen3-1.7B and vary a single component of the primary OASIS configuration. The default setup employs the shortest verified scaffold with $K=8$ rollouts and is denoted by $\dagger$ in both tables.

\paragraph{Scaffold selection rule.}
Verification identifies which problems receive supervision, whereas the scaffold selection rule determines which verified rollout serves as the supervision trajectory. We compare the shortest and longest verified rollouts while holding the set of supervised problems constant. Consequently, the two variants differ solely in the specific trajectory used for training: the longest rollout provides more supervised tokens per problem, while the shortest offers a more compact supervision path.

\begin{table}[ht]
\centering
\caption{Effect of scaffold selection on Qwen3-1.7B using AIME 2024. Both variants use the same verified rollouts and supervise the same set of problems. Results are reported as mean $\pm$ standard deviation over three evaluation seeds.}
\label{tab:ablation-selection}
\small
\begin{tabular}{lcccc}
\toprule
\rowcolor{headergray}
\textbf{Scaffold} & \textbf{Supervised length} & \textbf{Avg@12} & \textbf{Pass@12} & \textbf{Maj@12} \\
\midrule
Shortest verified$^{\dagger}$ & 603 & $54.17 \pm 0.58$ & $80.00 \pm 0.00$ & $68.88 \pm 1.57$ \\
Longest verified              & 909 & $52.33 \pm 0.51$ & $73.33 \pm 1.92$ & $60.00 \pm 1.92$ \\
\bottomrule
\end{tabular}
\end{table}

\paragraph{Number of rollouts.}
The number of sampled rollouts $K$ directly determines generation cost and the likelihood that a problem yields at least one verified trajectory. Lower values of $K$ reduce computational cost but risk leaving more problems without a valid scaffold. Higher values increase inference overhead while supplying more candidate trajectories for scaffold selection. We compare $K=4$ and $K=8$, reporting total generation cost, and downstream benchmark metrics.

\begin{table}[ht]
\centering
\caption{Effect of the number of rollouts on Qwen3-1.7B using AIME 2024. Generated tokens are measured over 100 optimizer steps. Results are reported as mean $\pm$ standard deviation over three evaluation seeds.}
\label{tab:ablation-k}
\small
\begin{tabular}{ccccc}
\toprule
\rowcolor{headergray}
$\boldsymbol{K}$ & \textbf{Generated tokens} &  \textbf{Avg@12} & \textbf{Pass@12} & \textbf{Maj@12} \\
\midrule
4             & 22.8M & $52.77 \pm 0.29$ & $76.67 \pm 0.33$ & $66.67 \pm 0.96$ \\
8$^{\dagger}$ & 43.5M &  $54.17 \pm 0.58$ & $80.00 \pm 0.00$ & $68.88 \pm 1.57$ \\
\bottomrule
\end{tabular}
\end{table}

\subsection{Implementation Details}
\label{app}

\begin{algorithm}[t]
\caption{One OASIS update}
\label{alg:oasis}
\begin{algorithmic}[1]
\Require problems $\{x_i\}$ with final answers $\{a_i\}$, policy $\pi_\theta$,
frozen teacher $\pi_{\theta_0}$, rollouts per problem $K$
\For{each problem $x_i$}
\State $\mathcal{Y}(x_i) \gets \{y^{(k)}\}_{k=1}^{K}$, \quad $y^{(k)} \sim \pi_\theta(\cdot \mid s(x_i))$
\State $\mathcal{Y}^{+}(x_i) \gets \{y : \text{answer}(y) = a_i\}$
\If{$\mathcal{Y}^{+}(x_i)=\emptyset$} mask $x_i$ \textbf{and continue} \Comment{rare all-masked-batch exception below}\EndIf
\State $y^{\star}_i \gets$ shortest element of $\mathcal{Y}^{+}(x_i)$ \Comment{Eq.~\ref{eq:shortest}}
\State $r_i \gets$ shortest element of $\mathcal{Y}(x_i)\setminus\mathcal{Y}^{+}(x_i)$,
else another verified rollout \Comment{Eq.~\ref{eq:context}}
\State $q_t \gets \pi_{\theta_0}(\cdot \mid u(x_i,r_i), y^{\star}_{i,<t})$
\State $p_t \gets \pi_{\theta}(\cdot \mid s(x_i), y^{\star}_{i,<t})$
\EndFor
\State update $\theta$ with Eq.~\ref{eq:opsd} over the unmasked problems
\end{algorithmic}
\end{algorithm}

Table~\ref{tab:hparams} lists the training configuration, which is identical across methods. Three additional details of the released OPSD implementation are relevant to the reported results and are therefore kept unchanged across all experiments.

\paragraph{Loss normalization.}
Equation~\ref{eq:opsd} averages over supervised tokens in the micro-batch rather than over sequences. Problems with no verified rollout contribute no tokens to either the numerator or denominator. A micro-batch in which a problem is masked therefore produces zero loss and zero gradient.

%\paragraph{Right padding and position indices.}Prompts are right-padded to the longest prompt in each micro-batch, and the completion is appended after the padding. No explicit position indices are passed, so padded positions still advance the rotary position index. Consequently, each completion begins $\text{gap}=\max_i |u_i|-|u_j|$ positions later than it would in an unpadded sequence. This behavior is inherited from the released implementation and is identical across all experimental arms. The gap becomes larger when reference lengths vary more within a micro-batch.  

\paragraph{Masking.}
A problem without a verified rollout is excluded from the loss. Its tokens receive no supervision, and the problem is excluded from the normalization in Eq.~\ref{eq:opsd}. The only exception occurs when all problems in a micro-batch are masked. In this case, rather than discarding the update, we retain one problem and distill along its first rollout, even though that rollout is unverified. This exception is rare. In our Qwen3-1.7B run, it occurred for 1.6\% of all problems encountered during training, corresponding to 3.0\% of the problems that contributed a gradient. In addition, a runtime check verifies for every micro-batch that masked problems contribute no supervised tokens and that every training scaffold is verified, except for the logged all-masked exceptions.

\paragraph{Compute and Hardware Resources.}
All experiments were conducted using 2$\times$ NVIDIA H100 GPUs and 8$\times$ NVIDIA A100 (80GB) GPUs. The generation of rollouts and teacher logit evaluations during the OASIS training pipeline were parallelized across the available GPU nodes.  

\begin{table}[t]
\centering
\small
\caption{\textbf{Training configuration.} Hyperparameters are identical across model sizes}
\label{tab:hparams}
\begin{tabular}{lccc}
\toprule
\rowcolor{headergray}
\textbf{Hyperparameter} & \textbf{Qwen3-1.7B} & \textbf{Qwen3-4B} & \textbf{Qwen3-8B} \\
\midrule
\multicolumn{4}{l}{\textit{Optimization \& Batching}} \\
\hspace{0.5em} Optimizer steps               & 100 & 100 & 100 \\
\hspace{0.5em} Problems per step             & 64  & 64  & 64  \\
\hspace{0.5em} Problems trained per step     & $\leq 32$ & $\leq 32$ & $\leq 32$ \\
\hspace{0.5em} Rollouts per problem ($K$, OASIS) & 8  & 8   & 8   \\
\hspace{0.5em} Learning rate                 & \multicolumn{3}{c}{$5\times10^{-6}$, linear decay, gradient clip $\|g\| \leq 0.1$} \\
\hspace{0.5em} Clip constant $\kappa$        & \multicolumn{3}{c}{$0.05$, per vocabulary entry} \\
\cmidrule(lr){1-4}
\multicolumn{4}{l}{\textit{Model Architecture \& Distillation}} \\
\hspace{0.5em} LoRA rank / $\alpha$           & \multicolumn{3}{c}{$64 \:/\: 128$, all attention and MLP projections} \\
\hspace{0.5em} Distillation temperature      & \multicolumn{3}{c}{$1.1$} \\
\hspace{0.5em} Teacher                       & \multicolumn{3}{c}{Initial policy $\pi_{\theta_0}$ (frozen, LoRA disabled)} \\
\hspace{0.5em} Student / Teacher template    & \multicolumn{3}{c}{Thinking disabled / Thinking enabled} \\
\cmidrule(lr){1-4}
\multicolumn{4}{l}{\textit{Generation \& Evaluation}} \\
\hspace{0.5em} Rollout sampling               & \multicolumn{3}{c}{$T = 1.1$, top-$p = 0.95$, top-$k = 20$, $\leq 1024$ new tokens} \\
\hspace{0.5em} Evaluation metrics             & \multicolumn{3}{c}{\text{Avg}@12, $T = 1.0$, top-$p = 0.95$, $\leq 38{,}912$ tokens} \\
\bottomrule
\end{tabular}
\end{table}

\end{document}